\documentclass{article}
\usepackage{nips14submit_e,times}
\usepackage{amsmath}
\usepackage{amssymb}
\usepackage{epsfig}
\usepackage{graphicx}
\usepackage{amsmath,amsthm}
\usepackage{amssymb}
\usepackage{algorithm}
\usepackage{algpseudocode}
\usepackage{subcaption}
\usepackage{hyperref}
\usepackage[applemac]{inputenc}
\usepackage{authblk}
\usepackage{pst-sigsys}
\graphicspath{{../figures/}}

\begin{document}
\renewcommand{\baselinestretch}{1.1}
\newcommand{\rb}{\rangle}
\newcommand{\lb}{\langle}
\newcommand{\om}{\omega}
\newcommand{\Real} {{\rm Real}}
\newcommand{\R} {{\mathbb R}}
\newcommand{\V} {{\cal V}}
\newcommand{\N} {{\mathbb N}}
\newcommand{\HH} {{\cal H}}
\newcommand{\E} {{\mathbb E}}
\newcommand{\Ld} {{\bf L^2}}
\newcommand{\sign}{{\rm sign}}
\newcommand{\C}{{\mathbb C}}
\newcommand{\W}{{\bf W}}
\newtheorem{theorem}{Theorem}[section]
\newtheorem{corollary}{Corollary}[section]
\newtheorem{proposition}{Proposition}[section]
\newtheorem{lemma}{Lemma}[section]

%\nipsfinalcopy

\newcommand\blfootnote[1]{
  \begingroup
  \renewcommand\thefootnote{}\footnote{#1}
  \addtocounter{footnote}{-1}
  \endgroup
}

\nipsfinalcopy

\author[1,2]{Xu Chen}
\author[2]{Xiuyuan Cheng}
\author[2]{St\'ephane Mallat}
\affil[1]{Department of Electrical Engineering,
Princeton University, NJ, USA}
\affil[2]{D\'epartement  d'Informatique, \'Ecole Normale Sup\'erieure, Paris, France
}

\title{Unsupervised Deep Haar Scattering on Graphs}

\maketitle

%\vspace{-0.25cm}
\begin{abstract}
The classification of high-dimensional data defined on graphs 
is particularly difficult when the graph geometry is unknown. 
We introduce a Haar scattering transform on graphs, which computes 
invariant signal descriptors. It is implemented 
with a deep cascade of additions, subtractions and absolute values,
which iteratively compute orthogonal Haar wavelet transforms.
Multiscale neighborhoods of unknown graphs are
estimated by minimizing an average total variation, 
with a pair matching algorithm of polynomial
complexity. Supervised classification with dimension reduction
is tested on data bases of scrambled images, and for signals 
sampled on unknown irregular grids on a sphere. 
\blfootnote{This work was supported by the ERC grant InvariantClass 320959.}
\vspace{0.5cm}
\end{abstract}

%%%%%%%%%%%%%%%%%%%%%%%%%%%%%%%%%%%%%%%%%%%%%
\section{Introduction}

The geometric structure of a data domain can be described with a 
graph \cite{graphs}, where neighbor data points are represented by vertices related by an edge. For sensor networks, this connectivity  depends upon the sensor physical locations, but in social networks it may correspond to strong interactions or similarities between two nodes. In many applications, the connectivity graph is unknown and  must therefore be estimated from data. We introduce an unsupervised learning algorithm to classify signals defined on an unknown graph. 

An important source of variability on graphs results from displacement of signal values. It may be due to movements of physical sources in a sensor network, or to propagation phenomena in social networks.  Classification problems are often
invariant to such displacements. Image pattern recognition or characterization of communities in social networks are examples of invariant problems. They 
require to compute locally or globally invariant descriptors, which are sufficiently rich to discriminate complex signal classes. 

Section \ref{Orhanssec} introduces a Haar scattering transform which builds an invariant representation of graph data, by cascading 
additions, subtractions and absolute values in a deep network.  
It can be factorized as a product of Haar wavelet transforms on the graph.
Haar wavelet transforms are flexible representations which 
characterize multiscale signal patterns on graphs \cite{gavish10,Guibas,graphs}. 
Haar scattering transforms are extensions on graphs of
wavelet scattering transforms, previously introduced for uniformly
sampled signals \cite{mallat}.

For unstructured signals defined on an unknown graph, 
recovering the full graph geometry
is an NP complete problem. We avoid this complexity by only learning 
connected multiresolution graph approximations. This is sufficient to compute  Haar
scattering representations. Multiscale neighborhoods are calculated
by minimizing an average total signal variation over training examples. It involves 
a pair matching algorithm of polynomial complexity. We show that
this unsupervised learning algorithms computes
sparse scattering representations.

For classification, the dimension of unsupervised
Haar scattering representations are reduced with supervised partial least
square regressions \cite{PLS-review}. It amounts
to computing a last layer of reduced dimensionality,
before applying a Gaussian kernel SVM classifier. 
The performance of a Haar scattering classification is tested on scrambled images, whose graph geometry is unknown. Results are provided for MNIST and CIFAR-10 image data bases.  Classification experiments are also performed  on scrambled signals whose samples are on an irregular grid of a sphere.
All computations can be reproduced with a software available at {\it www.di.ens.fr/data/scattering/haar}.

\section{Orthogonal Haar Scattering on a Graph}
\label{Orhanssec}

\begin{figure}
\begin{center}
\includegraphics[width=0.75\linewidth]{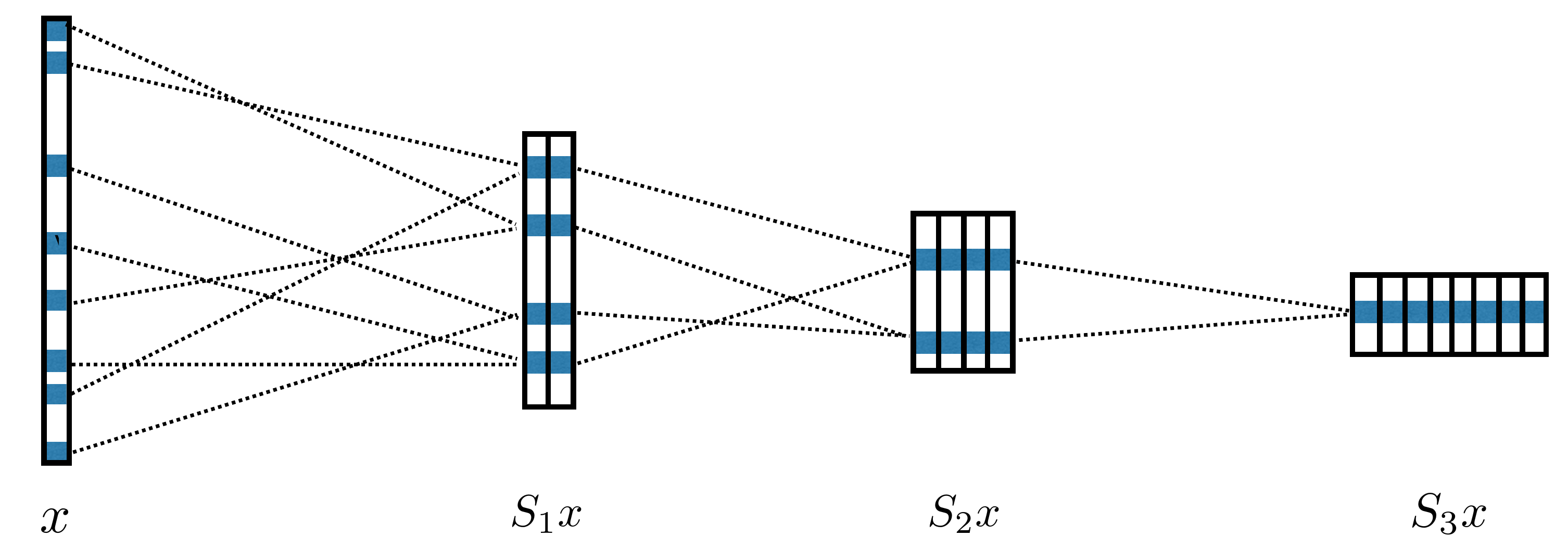}
\caption{\label{fig:1}
\small A Haar scattering network computes each coefficient
of a layer $S_{j+1}x$ by adding or subtracting a pair of coefficients
in the previous layer $S_j x$.
}
\end{center}
\end{figure}

\subsection{Deep Networks of Permutation Invariant Operators}

We consider signals  $x$ defined on an unweighted graph $G=(V,E)$, with
$V=\{1, ..., d\}$. Edges relate neighbor vertices.
We suppose that $d$ is a power of $2$ to simplify explanations. 
A Haar scattering is calculated by iteratively applying the following
permutation invariant operator
\begin{equation}
\label{permasn}
(\alpha,\beta) ~\longrightarrow
(\alpha+\beta,|\alpha-\beta|)~. 
\end{equation}
Its values are not modified by  a permutation of $\alpha$ and $\beta$, and both values
are recovered by 
\begin{equation}\label{eq:recover}
\max(\alpha,\beta) = \frac 1 2 \big(\alpha+\beta+|\alpha-\beta| \big)~~\mbox{and}~~
\min(\alpha,\beta) = \frac 1 2 \big(\alpha+\beta-|\alpha-\beta| \big)~.
\end{equation}

An orthogonal Haar scattering transform computes progressively more invariant
signal descriptors by applying this invariant operator
at multiple scales. This is implemented along a deep network
illustrated in Figure \ref{fig:1}. 
The network layer $j$ is a two-dimensional array 
$S_j x (n,q)$ of $d = 2^{-j} d \times 2^j$ coefficients, 
where $n$ is a node index and $q$  is a feature type.

The input network layer is $S_0 x (n,0) = x(n)$. We compute $S_{j+1} x$ by regrouping the $2^{-j} d$ nodes of $S_{j} x$  in $2^{-j-1} d$ pairs $(a_{n},b_{n})$, and applying the permutation invariant
operator (\ref{permasn}) to each pair $(S_j x(a_n,q), S_j x(b_n,q))$:
\begin{equation}
\label{eqn1}
S_{j+1} x (n,2q) = S_j x(a_n,q) + S_j x(b_n,q)
\end{equation}
and
\begin{equation}
\label{eqn2}
S_{j+1} x(n,2q+1) = |S_j x(a_n,q) - S_j x(b_n,q)|~.
\end{equation}
This transform is iterated up to a maximum depth $J \leq \log_2(d)$. It
computes $S_J x$ with $J d/2$ additions, subtractions and 
absolute values. Since $S_j x \geq 0$ for $j > 0$,  one can put an absolute value on the sum in (\ref{eqn1}) without
changing $S_{j+1} x$. It results that $S_{j+1} x$ is calculated from the previous
layer $S_j x$ by applying a linear operator followed by a non-linearity 
as in most deep neural network architectures. In our case this non-linearity is
an absolute value as opposed to rectifiers used in most deep networks 
\cite{Bengio}. 

For each $n$, the $2^j$ scattering coefficients $\{S_j x(n,q)\}_{0 \leq q < 2^j}$ 
are calculated
from the values of $x$ in a vertex set $V_{j,n}$ of size $2^j$.
One can verify by induction on (\ref{eqn1}) and (\ref{eqn2}) that 
$V_{0,n} = \{n\}$ for $0 \leq n < d$, and for any $j \geq 0$
\begin{equation}
\label{regroup}
V_{j+1,n} = V_{j,a_n} \cup V_{j,b_n}~.
\end{equation}

The embedded subsets $\{ V_{j,n} \}_{j,n}$  form a multiresolution approximation  of the vertex set $V$. At each scale $2^j$,  different pairings $(a_n,b_n)$ define different multiresolution approximations. A small graph displacement propagates signal values from a node to its neighbors. To build nearly invariant representations over such displacements,  a Haar scattering transform must regroup connected vertices.  It is thus computed over multiresolution vertex sets $V_{j,n}$ which are connected in the graph $G$.  It results from (\ref{regroup}) that   a necessary and sufficient condition is that each pair $(a_n,b_n)$ regroups two connected sets $V_{j,a_n}$ and $V_{j,b_n}$. 

Figure \ref{fig:2} shows two examples of connected multiresolution approximations. Figure \ref{fig:2}(a) illustrates the graph of an image grid, where pixels are connected to $8$ neighbors. In this example, each $V_{j+1,n}$ regroups two subsets $V_{j,a_n}$ and $V_{j,b_n}$ which are connected horizontally if $j$ is even and connected vertically if $j$ is odd. Figure \ref{fig:2}(b) illustrates a second example of connected multiresolution approximation on an irregular graph. There are many different connected multiresolution approximations resulting from different pairings at each scale $2^j$. Different multiresolution approximations correspond to different Haar scattering transforms. In the following, we compute several Haar scattering transforms of a signal $x$, by defining different multiresolution approximations.

\begin{figure*}[t!]
    \centering
    \begin{subfigure}[b]{0.4\textwidth}
        \centering
        \includegraphics[width=0.8\linewidth]{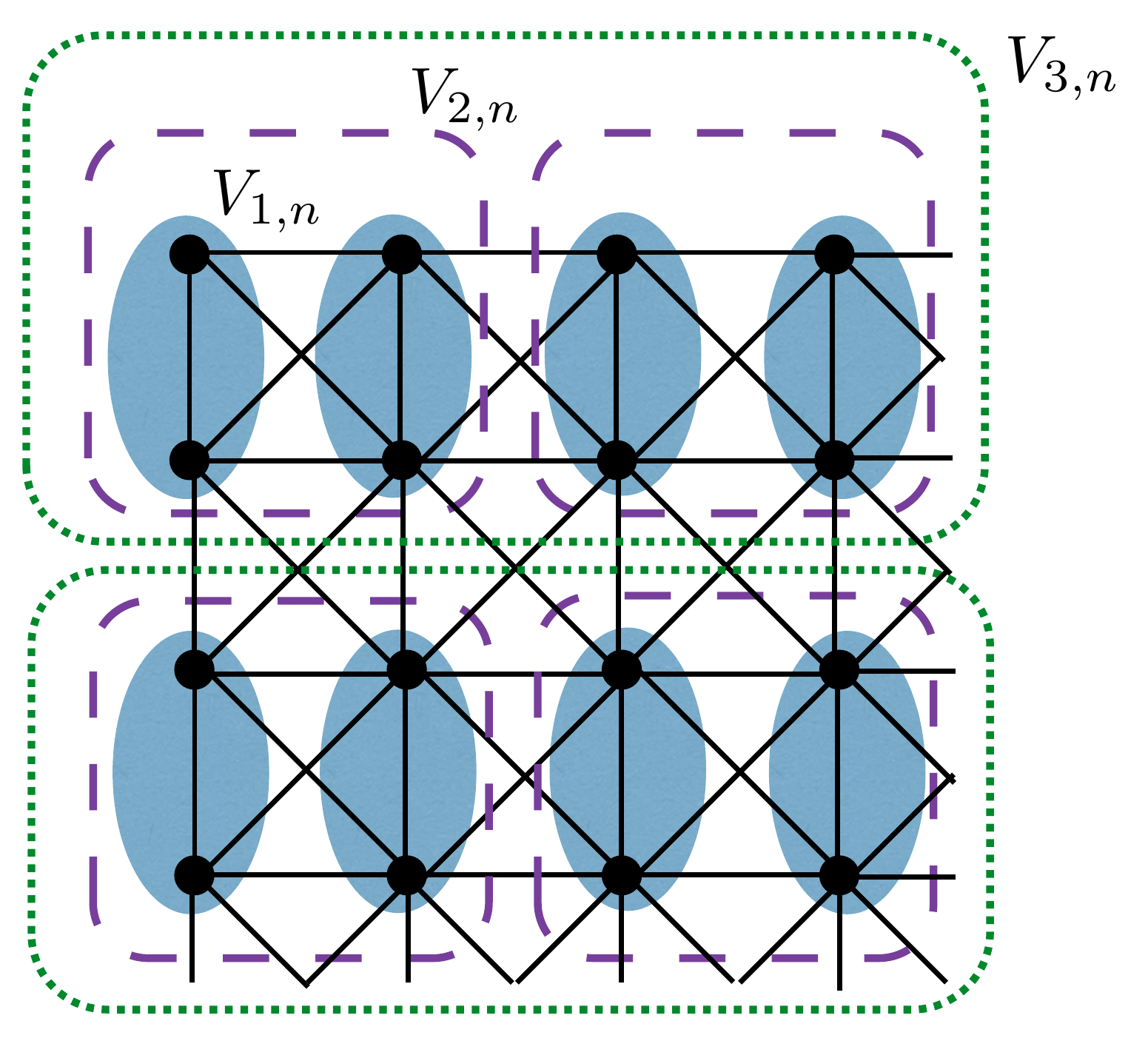} 
         \caption{ 
	\small
	}
    \end{subfigure}%
    ~~ 
    \begin{subfigure}[b]{0.4\textwidth}
        \centering
         \includegraphics[width=0.8\linewidth]{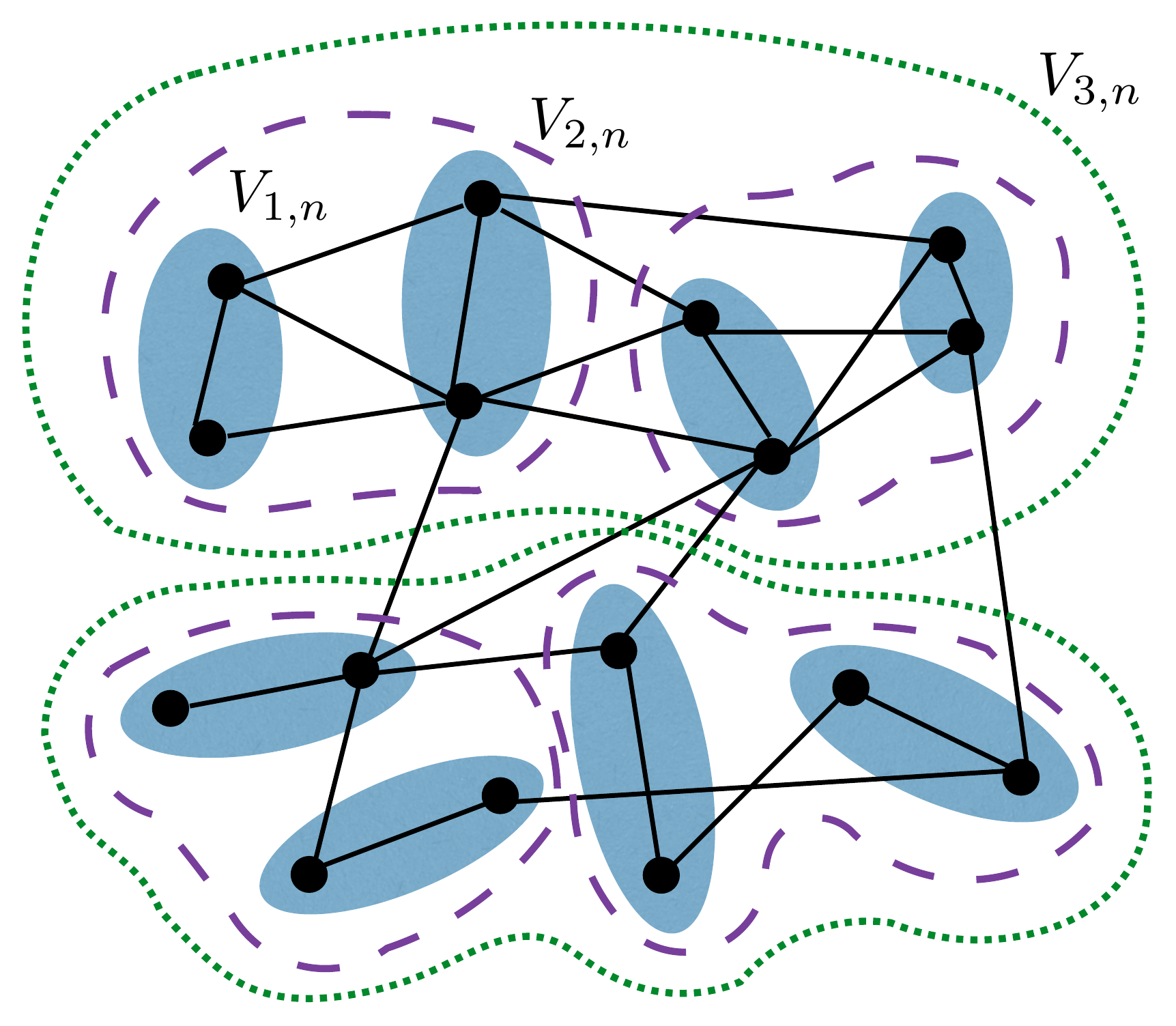} 
         \caption{ 
	\small
	}
    \end{subfigure}
 \caption{\label{fig:2}
\small A connected multiresolution is  a partition of vertices with embedded connected sets $V_{j,n}$ of size $2^j$. 
(a): Example of partition for the graph of a square image grid, for $1 \leq j \leq 3$. 
(b): Example on an irregular graph. 
 }
%\vspace{-0.1cm}
\end{figure*}

%% norm preserving
The following theorem proves that a Haar scattering preserves the norm and  that it is contractive up to a normalization factor $2^{j/2}$.  The contraction is due to the absolute value which suppresses the sign and hence reduces the amplitude of differences. The proof is in Appendix A.

\begin{theorem}\label{prop:norm-preserve}
For any $j \geq 0$, and any $x,x'$ defined on $V$
\[
\|S_j x - S_j x' \| \leq 2^{j/2} \|x - x' \|~,
\]
and 
\[
\|S_j x\| = 2^{j/2} \|x \|~.
\]
\end{theorem}

\subsection{Iterated Haar Wavelet Transforms}

We show that a Haar scattering transform can be written as a cascade of orthogonal Haar wavelet transforms and absolute value non-linearities. It is a particular example of scattering transforms introduced in \cite{mallat}. It computes coefficients measuring signal variations at multiple scales and multiple orders. We prove that the signal can be recovered from Haar scattering coefficients  computed over enough multiresolution approximations.

A scattering operator is contractive because of the absolute value.  When coefficients have an arbitrary sign, suppressing the sign reduces by a factor $2$ the volume of the signal space. We say that $S_J x(n,q)$ is a coefficient of order $m$ if its computation includes $m$ absolute values of differences. The amplitude of scattering coefficients typically decreases exponentially when the scattering order $m$ increases, because of the contraction produced by the absolute value.  We verify from (\ref{eqn1}) and (\ref{eqn2}) that $S_J x(n,q)$ is a  coefficient of order $m=0$ if $q=0$ and of order $m > 0$ if
\[
q = \sum_{k=1}^m 2^{J-j_k}~~\mbox{for}~~0 \leq j_k < j_{k+1} \leq J~.
\]
It results that there are ${J \choose m} 2^{-J} d $ coefficients $S_J x(n,q)$ of order $m$.

We now show that Haar scattering coefficients of order $m$  are obtained by cascading $m$ orthogonal Haar wavelet tranforms defined on the graph $G$. A Haar wavelet at a scale $2^J$ is defined over each $V_{j,n} = V_{j-1,a_n} \cup V_{j-1,b_n}$ by
\[
\psi_{j,n} = 1_{V_{j-1,a_n}} - 1_{V_{j-1,b_n}}~. 
\]
For any $J \geq 0$, one can verify \cite{Guibas,gavish10} that
\[
\{1_{V_{J,n}} \}_{0 \leq n < 2^{-J } d} \cup \{\psi_{j,n}\}_{0 \leq n < 2^{-j} d , 0 \leq j < J}
\]
is a non-normalized orthogonal Haar basis of the space of signals defined on $V$. Let us denote $\lb x , x' \rb = \sum_{v \in V} x(v)\, x'(v)$.  
Order $m=0$ scattering coefficients sum the values of $x$ in each $V_{J,n}$
\[
S_J x(n,0) = \lb  x\,,\,1_{V_{J,n}} \rb ~.
\]
Order $m=1$ scattering coefficients are sums of absolute values of
orthogonal Haar wavelet
coefficients. They measure the variation amplitude $x$
at each scale $2^{j}$, in each $V_{J,n}$:
\[
S_J x (n,2^{J-j_1}) = \sum_{p \atop{V_{j_1,p}\subset V_{J,n}}}
| \lb  x\,,\, \psi_{j_1, p} \rb |.
\]
Appendix B proves that second order scattering coefficients $S_J x(n,2^{J-j_1}+2^{J-j_2})$ are computed by applying a second orthogonal Haar wavelet transform to
first order scattering coefficients. 
A coefficient $S_J x(n,2^{J-j_1}+2^{J-j_2})$ 
is an averaged second order increment over $V_{J,n}$, calculated 
from the variations at the scale $2^{j_2}$,
of the increments of $x$ at the scale $2^{j_1}$. 
More generally, Appendix B also proves
that order $m$ coefficients 
measure multiscale variations of $x$ at the order $m$, and are obtained
by applying a Haar wavelet transform on scattering
coefficients of order $m-1$.

A single Haar scattering transform loses information since it applies a cascade of permutation invariant operators. However, the following theorem proves that $x$ can be recovered from scattering transforms
computed over $2^J$ different multiresolution approximations.

\begin{theorem}
\label{thm:multi-layer-recovery2}
There exist $2^J$ multiresolution approximations such that almost all $x \in {\mathbb R^d}$ can be reconstructed from
their scattering coefficients on these multiresolution approximations.
\end{theorem}

This theorem is proved in Appendix C. The key idea is that Haar
scattering transforms are computed with permutation invariants operators. Inverting these operators allows to recover values of signal pairs but not their locations. 
However, recombining these values on enough overlapping sets allows one to recover their locations and hence the original signal $x$. This is done with multiresolutions which are interlaced at each scale $2^j$, 
in the sense that if a multiresolution is pairing $(a_n,b_n)$ and $(a_n',b_n')$
then another multiresolution approximation is pairing $(a_n',b_n)$. 
Connectivity conditions are needed on the graph $G$ to guarantee the existence of ``interlaced'' multiresolution approximations which are all connected.

\section{Learning}

\subsection{Sparse Unsupervised Learning of Multiscale Connectivity}

Haar scattering transforms compute multiscale signal variations of multiple orders,
over non-overlapping sets of size $2^J$. To build signal descriptors which are
nearly invariant to signal displacements on the graph, we want to compute 
scattering transforms over connected sets in the graph, which a priori requires
to know the graph connectivity. However, in many applications,
the graph connectivity is unknown. For piecewise regular signals, 
the graph connectivity implies
some form of correlation between neighbor signal values, and may
thus be estimated from a training set of unlabeled examples $\{x_i \}_i$ \cite{roux2008learning}.

Instead of estimating the full graph geometry, which
is an NP complete problem, we estimate 
multiresolution approximations which are connected. 
This is a hierarchical clustering problem \cite{cluster}. 
A multiresolution approximation is connected
if at each scale $2^j$, each
pair $(a_n, b_n)$ regroups two vertex sets $(V_{j,a_n},V_{j,b_n})$ which are connected.
This connection is estimated by minimizing the total variation within 
each set $V_{j,n}$, which are clusters of size $2^j$ \cite{cluster}.
It is done with a fine to coarse aggregation strategy. 
Given $\{V_{j,n} \}_{0 \leq n < 2^{-j}d}$, we compute 
$V_{j+1,n}$ at the next scale,
by finding an optimal pairing $\{a_n,b_n\}_{n}$ which minimizes the 
total variation of scattering vectors, averaged over the
training set $\{x_i \}_i$:
\begin{equation}\label{eq:pairing-cost}
\sum_{n=0}^{2^{-j-1}d}~ \sum_{q=0}^{2^{j}-1}~ \sum_i 
|S_j x_i (a_n,q)-S_j x_i (b_n,q)|~.
\end{equation}
This is a weighted matching problem which can be solved by the Blossom Algorithm of Edmonds \cite{edmonds} with $O(d^3)$
operations. We use the implementation in \cite{rothberg}. 
Iterating on this algorithm for $0 \leq j < J$ thus computes a
multiresolution approximation at the scale $2^J$, with a 
hierarchical aggregation of graph vertices. 

Observe that 
\[
\|S_{j+1} x \|_1 = \|S_j x \|_1 + \sum_{q} \sum_{n}
|S_j x(a_n,q)-S_j x(b_n,q)|~.
\]
Given $S_j x$, it results that the minimization of (\ref{eq:pairing-cost}) is equivalent
to the minimization of $\sum_i \|S_{j+1} x_i \|_1$. This can be interpreted as
finding a multiresolution approximation which yields an optimally sparse scattering transform. It operates with a greedy layerwise strategy across
the network layers, similarly to sparse autoencoders 
for unsupervised deep learning \cite{Bengio}.

As explained in the previous section, several Haar scattering transforms are needed to obtain a complete signal representation. The unsupervised learning computes $N$ multiresolution approximations by dividing the training set $\{x_i \}_i$ in $N$ non-overlapping subsets, and  learning a  different multiresolution approximation from each training subset. % by minimizing the total variation (\ref{eq:pairing-cost}) at each scale $2^j$.

\subsection{Supervised Feature Selection and Classification}

The unsupervised learning computes a vector of scattering coefficients which is typically much larger than the dimension $d$ of $x$. However, only a subset of these invariants are 
needed for any particular classification task. The classification is improved by a supervised dimension reduction which selects a subset of scattering coefficients.  
In this paper, the feature selection is implemented with a partial least square regression \cite{PLS-review, Zhang2014, SchwartzICCV2009}. The final supervised classifier is a Gaussian kernel SVM.

Let us denote by $\Phi x = \{ \phi_p x \}_{p}$  the set of all scattering coefficients at a scale $2^J$,
computed from $N$ multiresolution approximations. 
We perform a feature selection adapted to each class $c$, with a
partial least square regression of the one-versus-all indicator function
\[
f_c(x) = 
\left\{
\begin{array}{ll}
1 & \mbox{if $x$ belongs to class $c$}\\
0 & \mbox{otherwise}
\end{array}
\right.~.
\]
A partial least square greedily selects and orthogonalizes 
each feature, one at a time. At the $k^{th}$ iteration, it selects 
a $\phi_{p_k} x$, and a Gram-Schmidt orthogonalization
yields a normalized $\tilde \phi_{p_k} x$, which is uncorrelated relatively to all previously
selected features:
\[
\forall r < k~~,~~\sum_i \tilde \phi_{p_k} (x_i)\, \tilde \phi_{p_r} (x_i) = 0
~~\mbox{and}~~\sum_i |\tilde \phi_{p_k} (x_i)|^2 = 1\, .
\]
The $k^{th}$ feature $\phi_{p_k} x$ is selected so that the linear regression of 
$f_c (x)$ on $\{ \tilde \phi_{p_r} x \}_{1 \leq r \leq  k}$ has a minimum mean-square error,
computed on the training set. This is equivalent to finding $\phi_{p_k}$ so that
$\sum_i f_c (x_i) \, \tilde \phi_{p_k} (x_i)$ is maximum.

The partial least square regression thus selects and computes
$K$ decorrelated scattering  features $\{ \tilde \phi_{p_k} x \}_{k < K}$ for each class $c$. 
For a total of $C$ classes, the union of all these feature sets defines 
a dictionary of size $M = K\, C$. They are linear combinations of the original
Haar scattering coefficients $\{ \phi_p x\}_p$. This dimension reduction 
can thus be interpreted as a last fully connected network layer, which outputs
a vector of size $M$. 
%A supervised Gaussian kernel SVM classifier is applied to this $M$ dimensional vector. 
The parameter $M$ allows one to optimize the 
bias versus variance trade-off. It can be adjusted from the decay of
the regression error of each $f_c$ \cite{PLS-review}.
In our numerical experiments, it is set to a fixed size for all data bases.

\section{Numerical Experiments}\label{sec:exp}

Unsupervised Haar scattering representations
are tested on classification problems,
over scrambled images and scrambled data on a sphere, for which the geometry
is therefore unknown. Classification results are compared with a Haar scattering
algorithm computed over the known signal geometry, and with state of the art
algorithms. 

A Haar scattering representation involves few parameters which are reviewed.
The scattering scale $2^J \leq d$ is the invariance scale. Scattering coefficients are computed up to the a maximum order 
$m$, which is set to $4$ in all experiments. Indeed, higher order scattering coefficient have a negligible relative energy, which is below $1\%$.
The unsupervised learning algorithm computes $N$ multiresolution approximations,
corresponding to $N$ different scattering transforms. Increasing $N$ decreases
the classification error but it increases computations.
The error decay becomes negligible for $N \geq 40$. 
The supervised dimension reduction selects a final 
set of $M$ orthogonalized scattering coefficients. We set  $M = 1000$ in all numerical experiments. 

For signals defined on an unknown graph, 
the unsupervised learning computes an estimation of connected 
multiresolution sets by minimizing an average total variation. 
For each data basis of scrambled signals,
the precision of this estimation is evaluated by computing the percentage of
multiscale sets which are indeed connected in the original topology (an image grid or a grid on the sphere).

\subsection{MNIST Digit Recognition}
\label{original}

\begin{figure}[t]
\begin{center}
\begin{subfigure}[b]{0.95\linewidth}
\includegraphics[width=0.09\linewidth]
{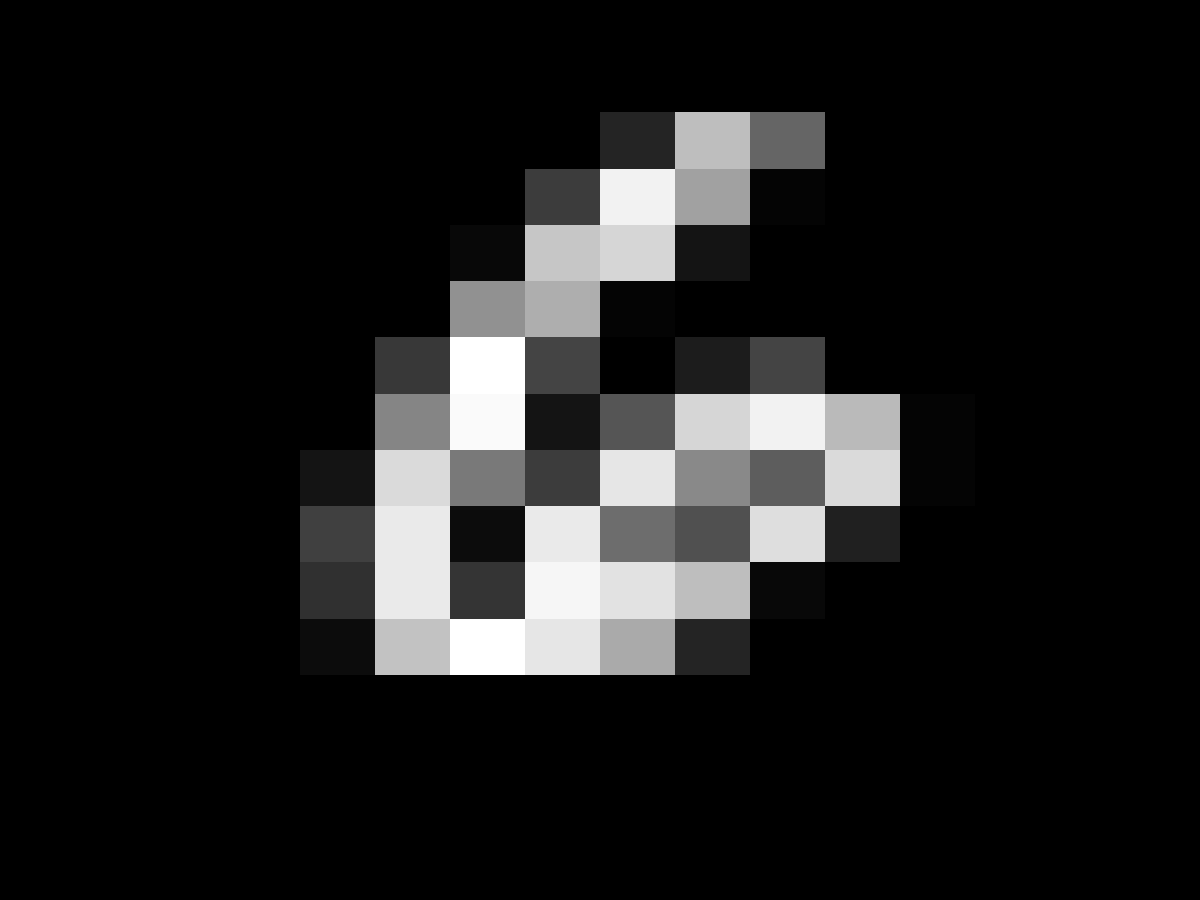}
\includegraphics[width=0.09\linewidth]
{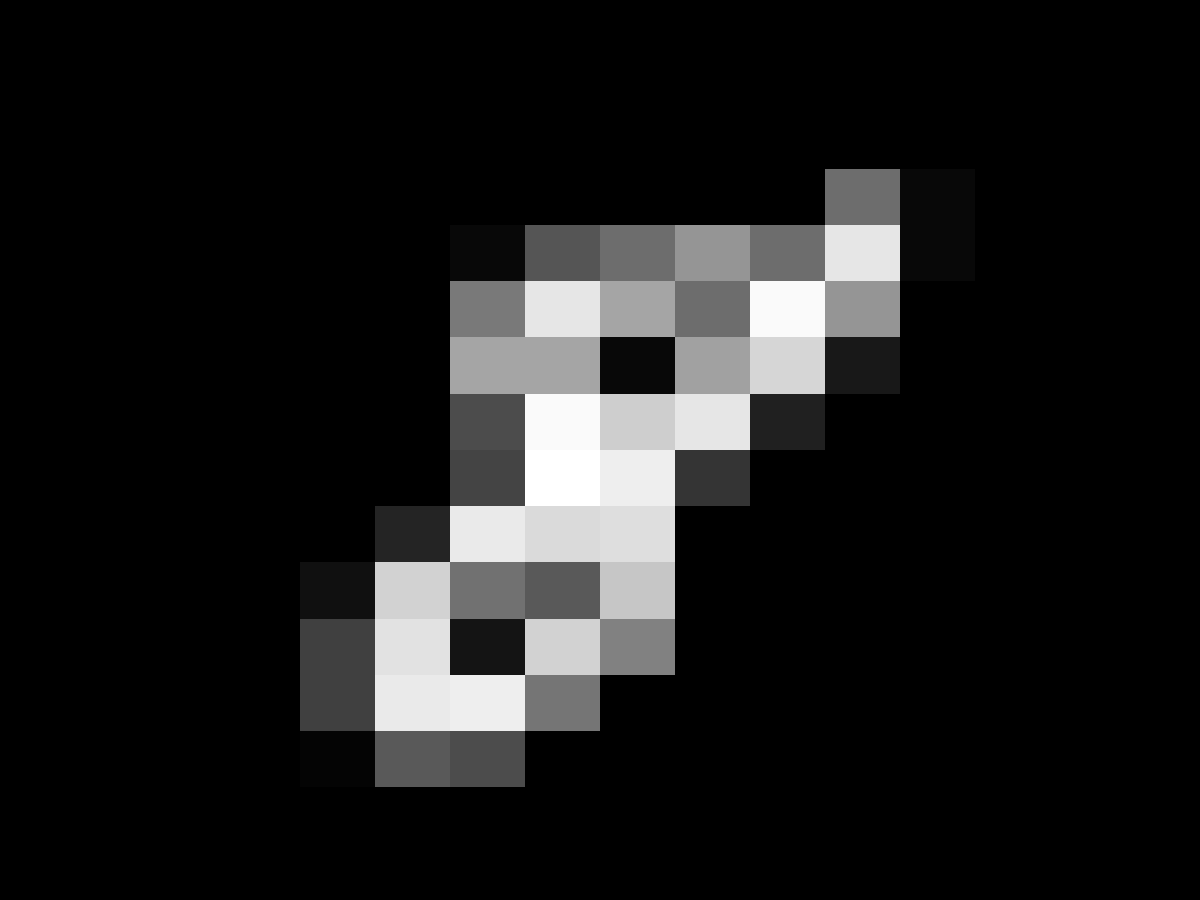}
\includegraphics[width=0.09\linewidth]
{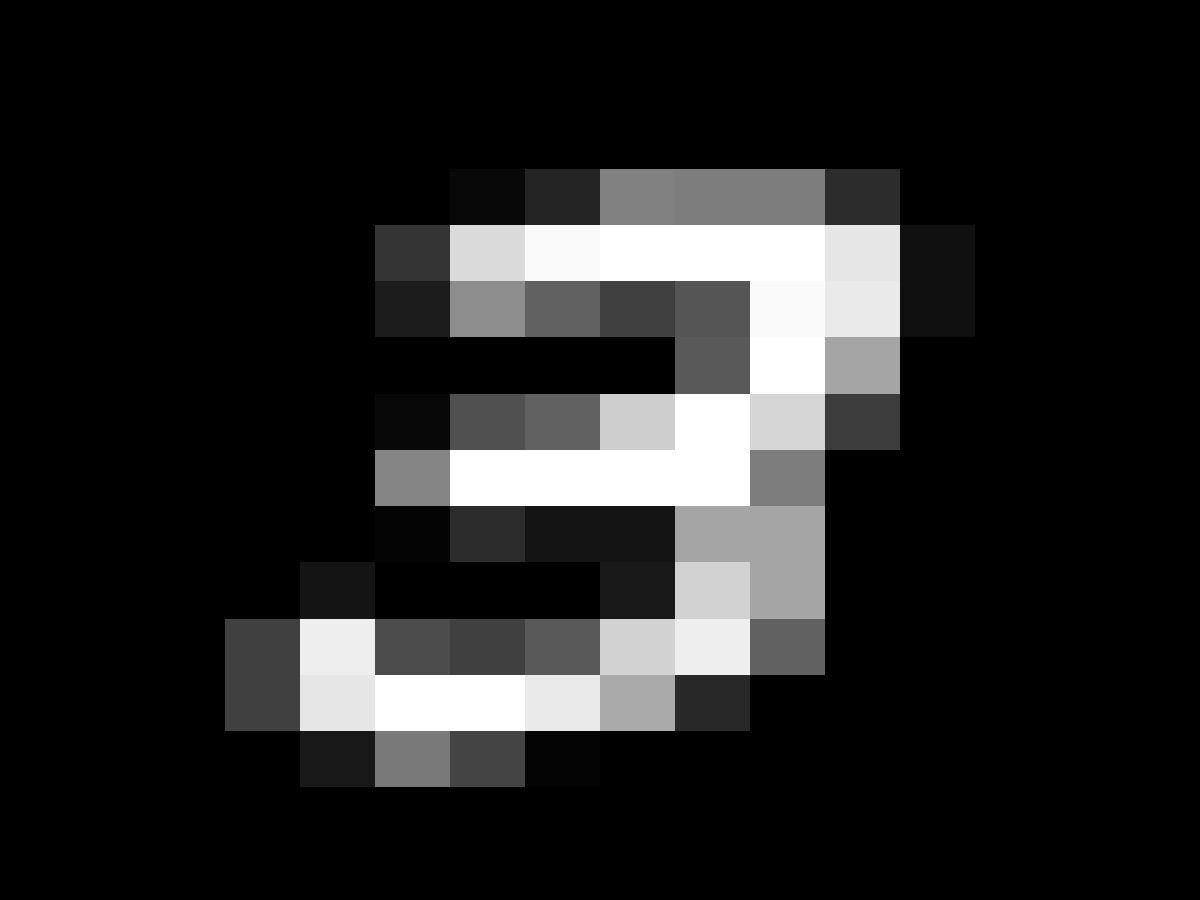}
\includegraphics[width=0.09\linewidth]
{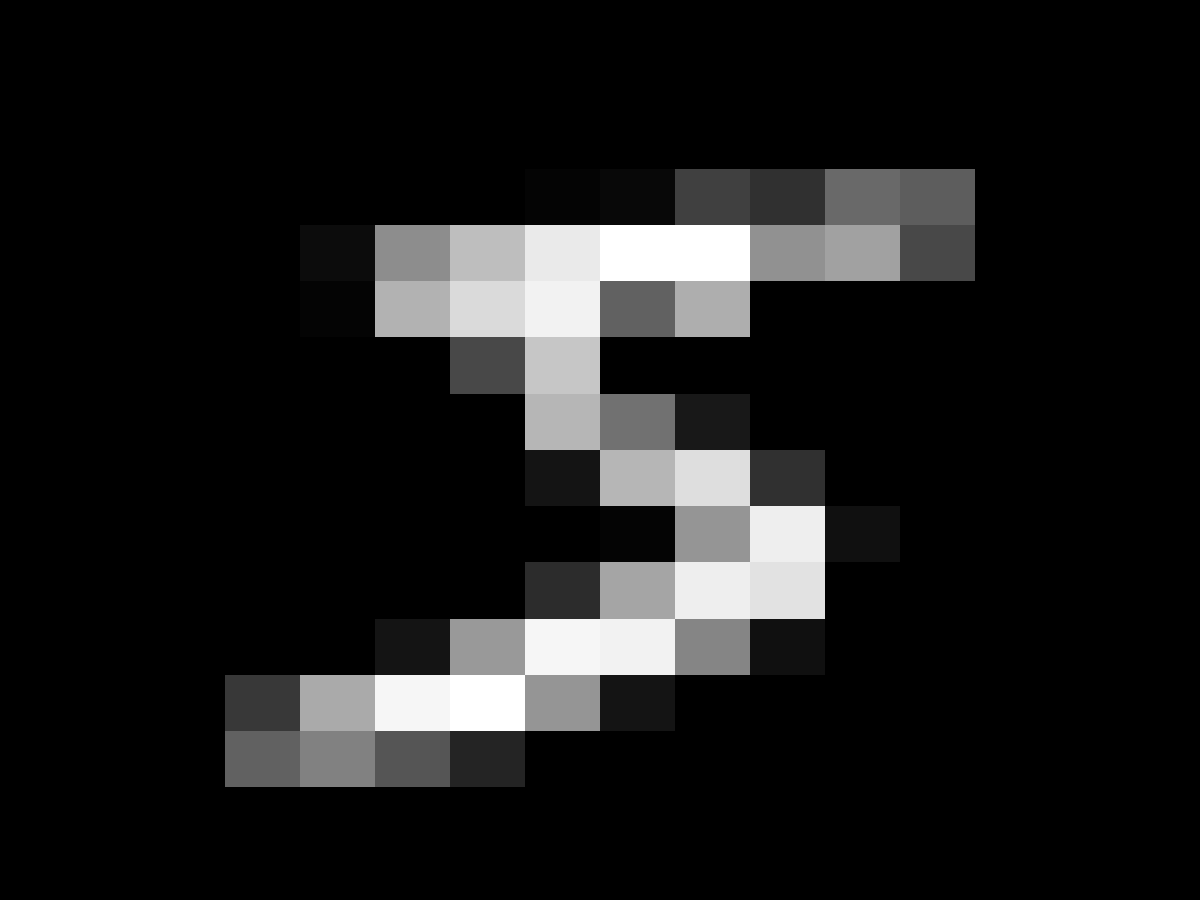}
\includegraphics[width=0.09\linewidth]
{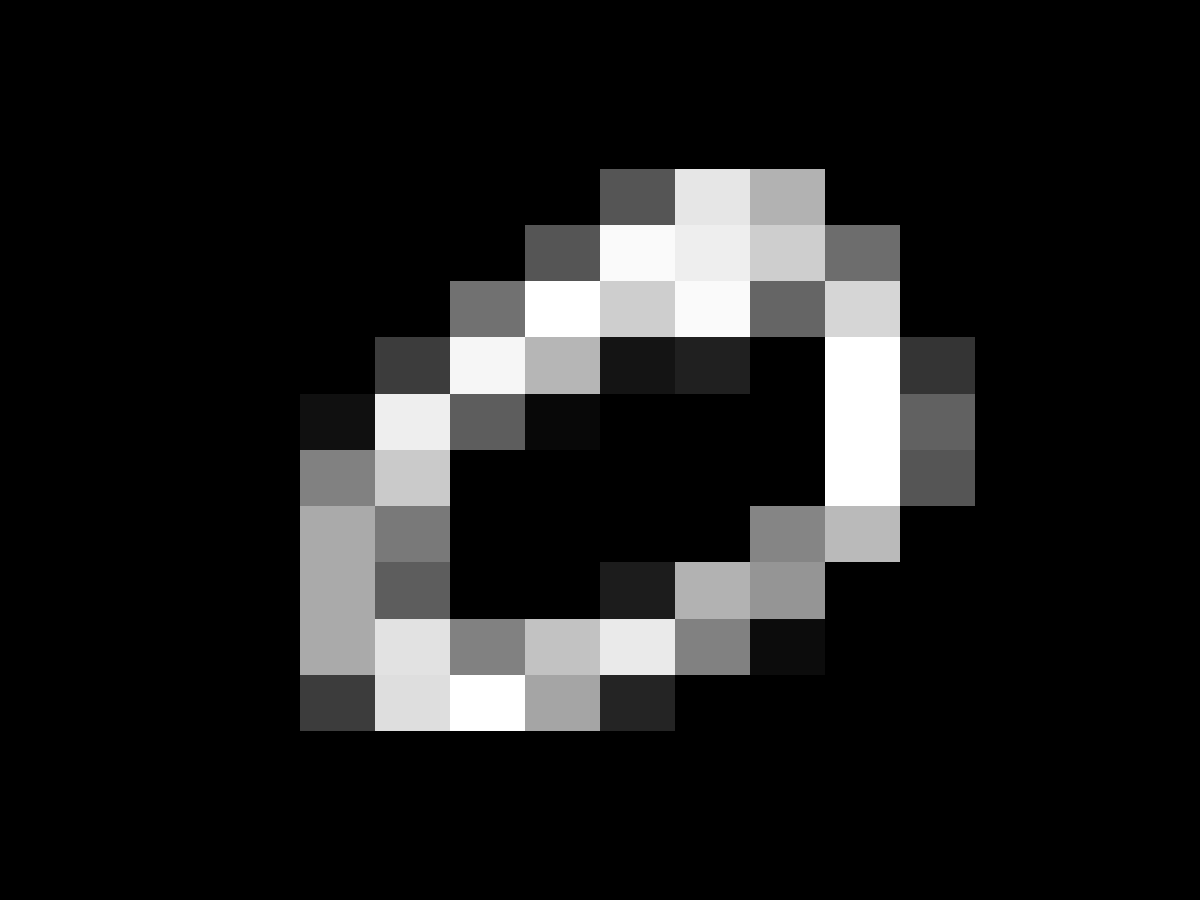}
\hspace{0.6cm}
\includegraphics[width=0.09\linewidth]
{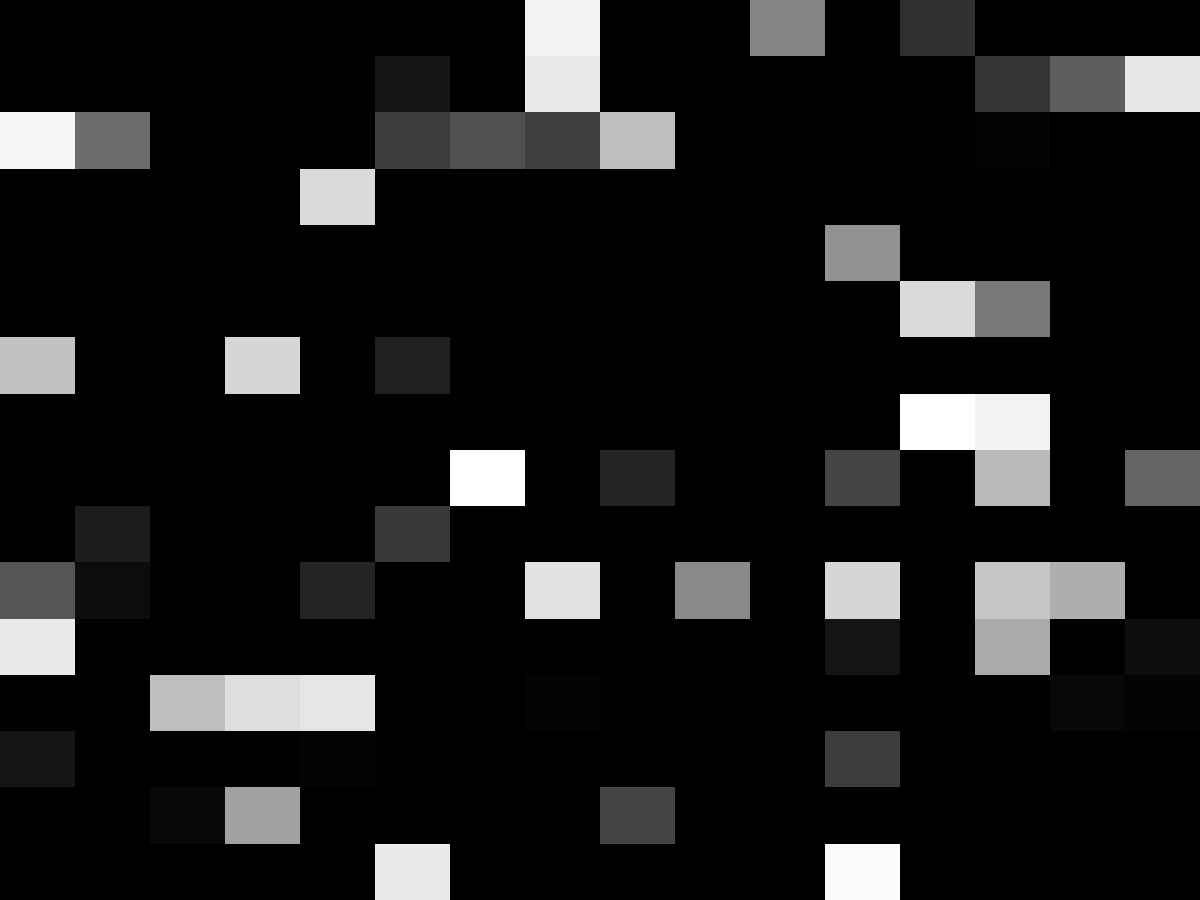}
\includegraphics[width=0.09\linewidth]
{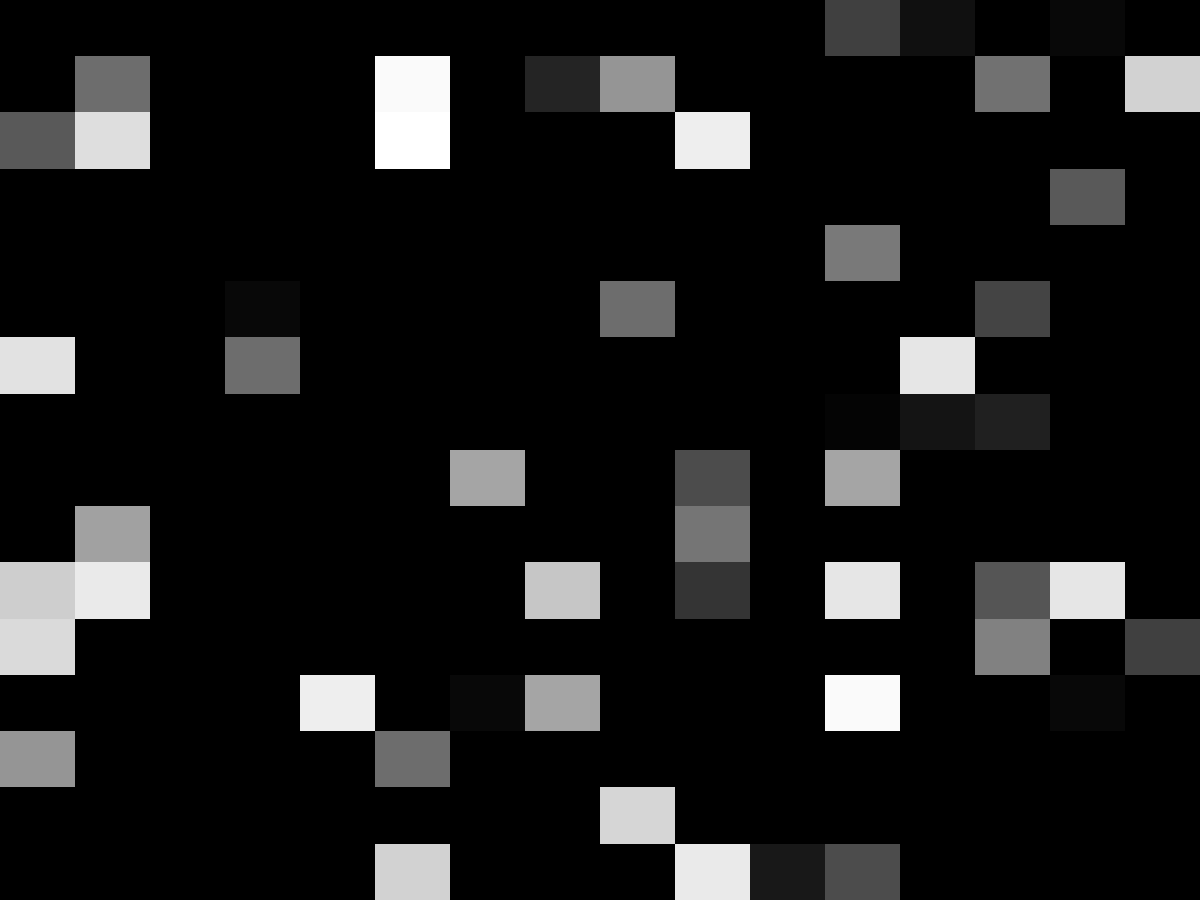}
\includegraphics[width=0.09\linewidth]
{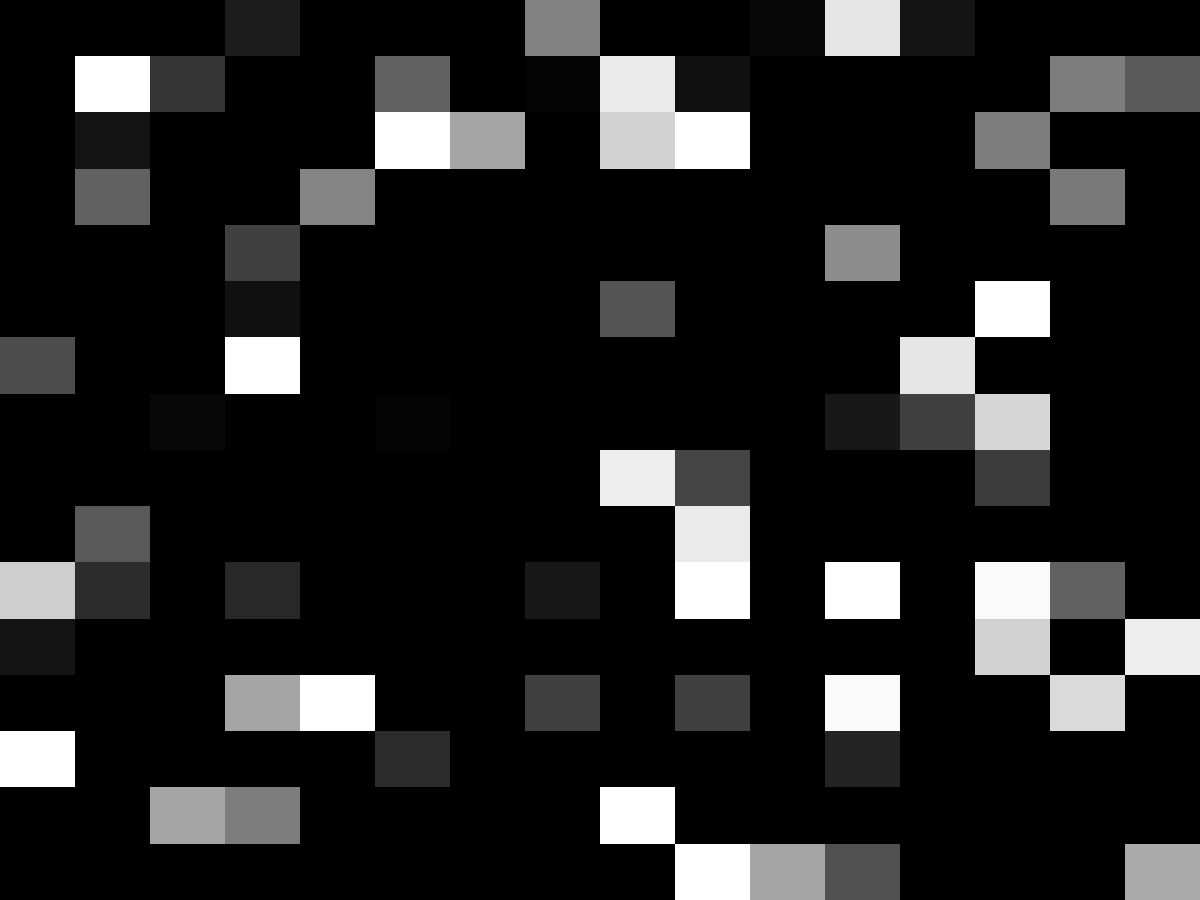}
\includegraphics[width=0.09\linewidth]
{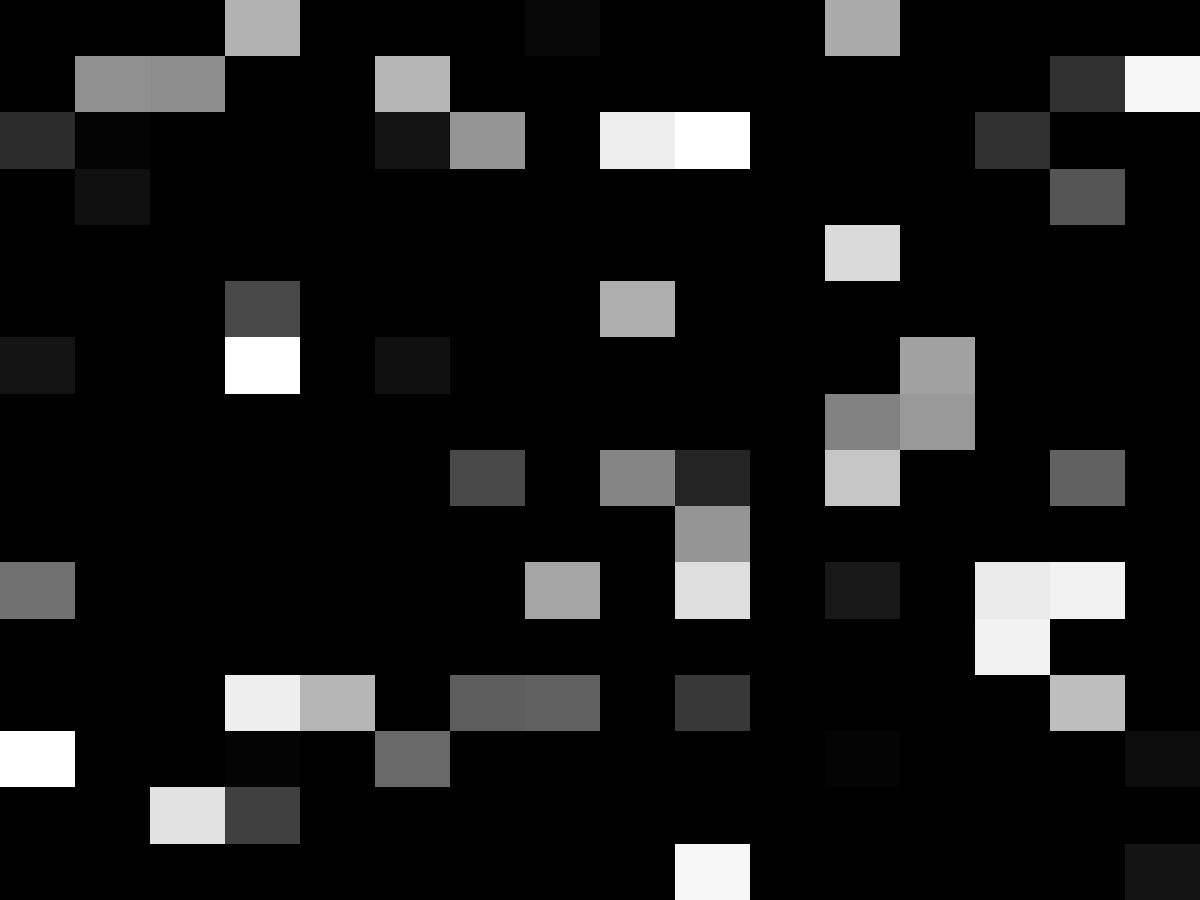}
\includegraphics[width=0.09\linewidth]
{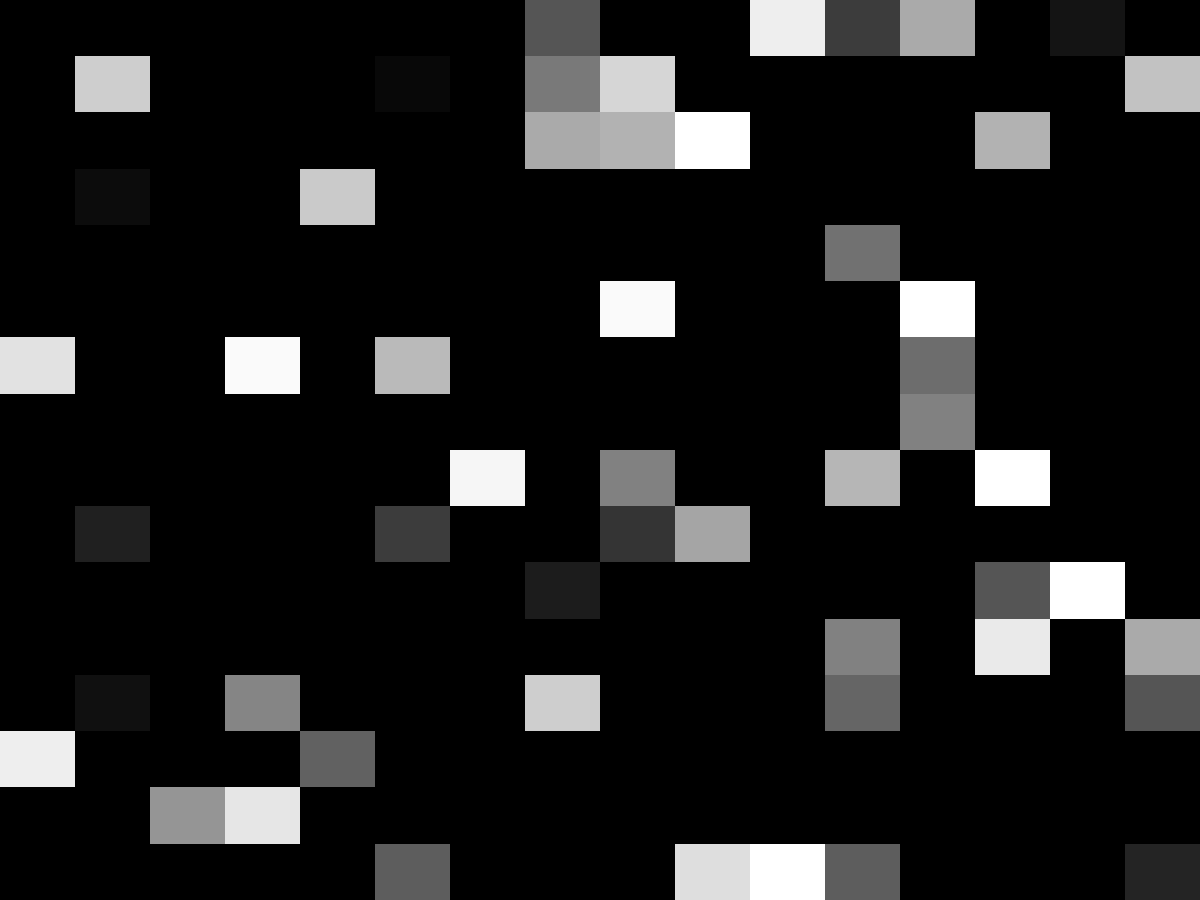}\vspace{-0.25cm}
\vspace{0.15cm}
\end{subfigure}
\caption{ 
\small MNIST images (left) and images after random pixel 
permutations (right).}
\label{fig:samplemnist}
%\vspace{-0.25cm}
\end{center}
\end{figure}

MNIST is a data basis with $6\times 10^4$ hand-written digit images of
size $d \leq 2^{10}$, with $5\times 10^4$ images for
training and $10^4$ for testing. Examples of MNIST images before and after
pixel scrambling are shown in Figure \ref{fig:samplemnist}. 
The best classification results are obtained with a maximum invariance scale $2^J = 2^{10}$. 
The classification error is $0.9\%$, with an unsupervised learning of $N=40$ multiresolution approximations.
Table \ref{table:mnist} shows that it is below but close to state of the art results obtained with fully supervised deep convolution, which are optimized with supervised backpropagation algorithms.

The unsupervised learning  computes multiresolution sets $V_{j,n}$ from scrambled images.
At scales $1 \leq 2^j \leq 2^3$, 
$100\%$ of these multiresolution sets 
are connected in the original image grid,
which proves that the geometry is well estimated at these scales. 
This is only evaluated on meaningful pixels which do not remain zero on all training images. 
For $j = 4$ and $j=5$ the percentages of connected sets 
are respectively $85\%$ and $67\%$. The percentage of connected sets 
decreases because long range correlations are weaker. 

One can reduce the Haar scattering classification error from $0.9\%$ to $0.59\%$ with a known image geometry. The Haar scattering transform is then
computed over multiresolution approximations which are directly constructed from the image grid as in Figure \ref{fig:2}(a).
Rotations and translations define $N = 64$ different connected multiresolution approximations, which yield a reduced error of $0.59\%$. State of the art classification errors on MNIST,
for non-augmented data basis (without elastic deformations), are respectively
$0.46\%$ with a Gabor scattering \cite{Joan} and $0.53\%$ with a supervised training of deep convolution networks \cite{LeCun}. This shows that without any learning,
a Haar scattering using geometry is close to the state of the art.
 
\begin{table}[h]
\begin{center}
\begin{tabular}{c|c|c|c}
\hline
Maxout MLP + dropout \cite{Goodfellow}
& Deep convex net. \cite{Yu} 
& DBM + dropout \cite{Hinton} 
& Haar Scattering\\
\hline
0.94 & 0.83 & \textbf{0.79} & 0.90 \\
\hline
\end{tabular}
\end{center}\vspace{-0.25cm}
\caption{\small 
Percentage of errors for the classification of scrambled MNIST images, obtained by different algorithms.}
\label{table:mnist}
%\vspace{-0.25cm}
\end{table}

\subsection{CIFAR-10 Images}\label{cifar}

CIFAR-10 images are color images of $32\times32$ pixels, 
which are much more complex than MNIST digit images.
It includes $10$ classes, such as ``dogs'', ``cars'', ``ships''
with a total of $5 \times 10^4$ training examples and $10^4$ testing examples. 
The $3$ color bands are represented with $Y,U,V$ channels and scattering 
coefficients are computed independently in each channel. 

The Haar scattering is first applied to scrambled CIFAR images whose geometry is unknown. The minimum classification error is obtained at
the scale $2^J = 2^7$ which is below the maximum scale $d = 2^{10}$. 
It maintains some localization information on the image features. 
With $N = 10$ multiresolution approximations, a Haar scattering transform
has an error of $27.3\%$. It is $10\%$ below previous results obtained on this
data basis, given in Table \ref{table:cifar}.

Nearly $100\%$ of the multiresolution sets $V_{j,n}$ computed from scrambled images are connected in the original image grid, for $1 \leq j \leq 4$, which
shows that the multiscale geometry is well estimated at these fine scales.  
For $j=5, 6$ and $7$, the proportions of connected sets are
$98\%$, $93\%$ and $83\%$ respectively. As for MNIST images, the connectivity
is not as precisely estimated at large scales. 

\begin{table}[h]
\begin{center}
\begin{tabular}{c|c|c}
\hline
 Fastfood \cite{Le} &Random Kitchen Sinks \cite{Le}& Haar Scattering\\
\hline
36.9  & 37.6 & \textbf{27.3} \\
\hline
\end{tabular}
\end{center}\vspace{-0.25cm}
\caption{\small Percentage of errors for the classification of scrambled 
CIFAR-10 images, with different algorithms.}
\label{table:cifar}
%\vspace{-0.25cm}
\end{table}

The Haar scattering classification error is reduced from $27.7\%$ to $21.3\%$ if
the image geometry is known. Same as for MNIST, we compute $N = 64$  multiresolution approximations obtained by translating and rotating. 
After dimension reduction, the classification error is $21.3\%$. This error is above the state
of the art obtained by a supervised convolutional network  \cite{Goodfellow} ($11.68\%$), but the Haar scattering representation involves no learning.

\subsection{Signals on a Sphere}
\label{3D}

A data basis of irregularly sampled signals on a sphere is constructed in \cite{Joan2}, by projecting the MNIST image digits on 
$d=4096$ points randomly sampled on the 3D sphere, and by randomly rotating these
images on the sphere. The random rotation is either uniformly distributed on the sphere 
or restricted with a smaller variance (small rotations) \cite{Joan2}.
The digit `9' is removed from the data set because it can not be
distinguished from a `6' after rotation. Examples of the dataset are shown in Figure \ref{fig:spheremnist}.

\begin{figure}[t]
\begin{center}
\includegraphics[width=0.24\linewidth]
{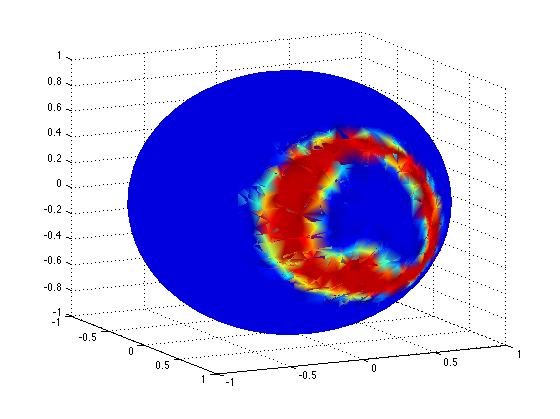}
\includegraphics[width=0.24\linewidth]
{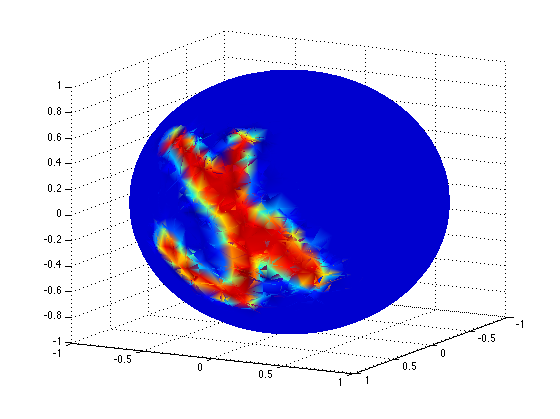}
\includegraphics[width=0.24\linewidth]
{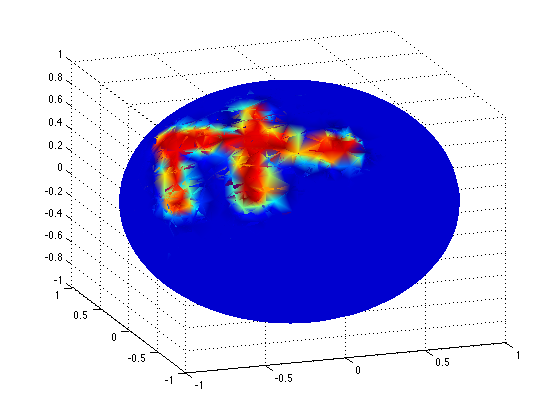}
\includegraphics[width=0.24\linewidth]
{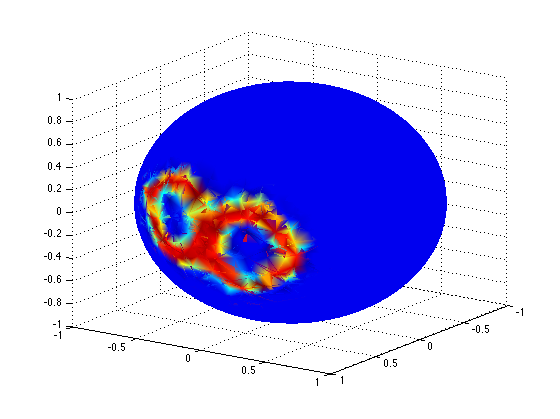}
%\vspace{-0.25cm}
\caption{\small Images of digits mapped on a sphere.}
\label{fig:spheremnist}
\end{center}
%\vspace{-0.15cm}
\end{figure}

The classification algorithms introduced in \cite{Joan2} use the known distribution of points on the sphere,
by computing a representation based on the graph Laplacian. 
Table \ref{table:smnist} gives the results reported in \cite{Joan2},
with a fully connected neural network, and a spectral graph Laplacian network.

As opposed to these algorithms, the Haar scattering algorithm uses no information on the positions of points on the sphere. Computations are
performed from a scrambled set of signal values, without any geometric information. Scattering transforms are calculated up to the maximum scale $2^J = d = 2^{12}$.  
A total of $N = 10$  multiresolution approximations are estimated by unsupervised learning, and the classification is performed from $M = 10^3$ selected
coefficients. Despite the fact that the geometry is unknown, the
Haar scattering reduces the error rate both for small and large 
3D random rotations. 

In order to evaluate the precision of our geometry estimation, we use the neighborhood information based on the 3D coordinates of the 4096 points on the sphere of radius 1. We say that two points are connected if their geodesic distance is smaller than 0.1. Each point on the sphere has on average $8$ connected points. For small rotations, the percentage of learned multiresolution sets which are connected is 92\%, 92\%, 88\% and 83\% for $j$ going from $1$ to $4$. It is computed on meaningful points with nonneglegible energy. For large rotations, it is 97\%, 96\%, 95\% and 95\%. This shows that the multiscale geometry on the sphere is well estimated.

\begin{table}[h]
\begin{center}
\begin{tabular}{c|c|c|c|c}
\hline 
&Nearest& Fully  & Spectral  & Haar \\
&Neighbors & Connect.  & Net.\cite{Joan2} & Scattering \\
\hline
Small rotations& 19 & 5.6 & 6 & \textbf{2.2}  \\
\hline
Large rotations& 80 & 52 & 50 & \textbf{47.7}  \\
\hline
\end{tabular}
\end{center}\vspace{-0.25cm}
\caption{\small Percentage of errors for the classification of MNIST images
rotated and sampled on a sphere
\cite{Joan2}, with a nearest neighbor classifier, a
fully connected two layer neural network, 
a spectral network \cite{Joan2}, and 
a Haar scattering.}
\label{table:smnist}
%\vspace{-0.25cm}
\end{table}

%%%%%%%%%%%%%%%%%%%%%%%%%%%%%%%%%%%%%%%%%%%%%
\section{Conclusion}
%Should mention that the algo is simple with no architecture design involved.

A Haar scattering transform computes invariant data representations by iterating
over a hierarchy of permutation invariant operators, calculated 
with additions, subtractions and absolute values. The geometry of unstructured
signals is estimated with an unsupervised learning algorithm, which minimizes
the average total signal variation over multiscale neighborhoods. 
This shows that unsupervised deep learning can be implemented with a polynomial complexity algorithm.
The supervised classification includes a feature selection implemented with a partial least square regression. State of the art results have been shown on
scrambled images as well as random signals sampled on a sphere.
The two important parameters of this architecture are the network depth, which 
corresponds to the invariance scale, and the dimension reduction of the final layer, set to $10^3$ in all experiments. It can thus easily be applied to any data set. 

This paper concentrates on scattering transforms of real valued signals. 
For a boolean vector $x$, a boolean scattering transform is
computed by replacing the operator (\ref{permasn}) by 
a boolean permutation invariant
operator which transforms $(\alpha,\beta)$ into
$(\alpha \, {\text{or}}\, \beta\, , \,
\alpha \, {\text{xor}}\, \beta)$. Iteratively applying this operator
defines a boolean scattering transform $S_j x$ having
similar properties. 
\newpage

\bibliographystyle{plain}

\newpage
\section*{Appendix}

\appendix

\numberwithin{equation}{section}
\numberwithin{theorem}{section}

\section{Proof of Theorem 2.1}
\begin{proof} 
Observe that the permutation invariant operator which associates to $(\alpha_0,\beta_0)$ the values
\[
(\alpha_1,\beta_1) = (\alpha_0+\beta_0,|\alpha_0-\beta_0|)
\]
satisfies
\[
\alpha_1^2 + \beta_1^2 = 2 ( \alpha_0^2 + \beta_0^2).
\]
Moreover, if $(\alpha'_1,\beta'_1) = (\alpha'_0+\beta'_0,|\alpha'_0-\beta'_0|)$ then 
\[
(\alpha_1-\alpha_1')^2 + (\beta_1-\beta_1')^2 \leq 2 \Big( (\alpha_0-\alpha'_0)^2 + 
(\beta_0 - \beta_0')^2 \Big).
\]
Since $S_{j+1}x$ is computed by applying this operator to pairs of values
of $S_j x$, we derive that 
\[
\|S_{j+1} x \|^2 = 2 \|S_{j+1} x \|^2~~\mbox{and}~
\|S_{j+1} x - S_{j+1} x' \|^2 \leq 2\, \|S_{j} x - S_{j} x' \|^2~.
\]
Since $S_0 x = x$ and $S_0 x' = x'$, iterating on these two equations
proves Theorem 2.1.

%To prove the first equality (norm preservation), it suffices to show that \[
%|| S_{j+1} x ||^2 = 2 || S_j x||^2.
%\]
%Recall the recursive definition of $S_j x$ in Eq. (3, 4), and let \[
%A = ( S_{j+1}x(n, 2q) )_{q=1}^{2^{j-1}}, \quad
%B = ( S_{j+1}x(n, 2q+1) )_{q=1}^{2^{j-1}}, \quad
%a = ( S_{j}x(a_n, q) )_{q=1}^{2^{j-1}}, \quad
%b = ( S_{j}x(b_n, q) )_{q=1}^{2^{j-1}}, 
%\]
%which are all vectors in $\R^{2^{j-1}}$, then we have \[
%A = a + b, \quad B = |a-b|, 
%\]
%where the absolute value of a vector applies entry-wisely. Thus \[
%||S_{j+1}x||^2 = ||A||^2 + ||B||^2 
%	=|| (a + b) ||^2 + \| (a-b) \|^2 
%	= 2 ( \| a\|^2 + \| b\|^2 )
% 	= 2 \| S_j x \|^2.
%\]
%
%To show the second equality (contraction), it suffices to show that $\| S_{j+1} x - S_{j+1} x' \| \le \sqrt{2} \| S_j x - S_j x'\| $. Again by Eq. (3, 4), and define $A, B, a, b$ as above, and similarly $A', B', a', b'$ for $S_{j+1}x'$ and $S_{j}x'$, then we have
%\begin{equation*}
%\begin{split}
%|| S_{j+1} x - S_{j+1} x' ||^2 
%&= \| A - A' \|^2 + \|  B - B'\|^2  \\
%&= \| (a+b) - (a'+b')\|^2 +  \| |a-b| - |a'-b'|\|^2  \\
%&\le \| (a-a') + (b-b') \|^2 + \| (a-a') - (b-b') \|^2 \quad 
%		\text{by that $||a| - |b|| \le |a-b|$, $\forall a,b \in \R$} \\
%&= 2 (  \| a-a' \|^2 + \| b-b' \|^2 ) = 2 \| S_j x - S_j x' \|^2.
%\end{split}
%\end{equation*}
\end{proof}

\section{Haar Scattering from Haar Wavelets}
 
The following proposition proves that order $m+1$ scattering coefficients are computed by applying an orthogonal Haar wavelet transform to order $m$ scattering  coefficients. We also prove by induction on $m$ that
a scattering coefficient 
$S_j x(n,q)$ is of order $m$ if and only if $q = 2^j \kappa$ with 
\[
\kappa = \sum_{k=1}^{m} 2^{-j_k}
\]
for some $0 < j_1< ...<j_m  \le J$.
This property is valid for $m = 0$ and the following proposition shows
that if it is valid for $m$ then it is also valid for 
$m+1$ in the sense that an 
order $m+1$ coefficient is 
indexed by $\kappa + 2^{-j_{m+1}}$, and it is computed 
by applying an orthogonal Haar transform to order $m$ scattering coefficients
indexed by $\kappa$.

\begin{proposition}
For any $v \in V$ and $0 \le q < 2^j$ we write
\[
\overline{S}_j x ( v , q) = \sum_{n=0}^{2^{-j}d-1} S_j x(n, q)\, 1_{V_{j,n}}( v) . 
\]
For any $\kappa = \sum_{k=1}^{m} 2^{-j_k}$, any $j_{m+1} > j_m$
and $0 \le n < 2^{-j} d$, 
\begin{equation}
\label{propsdfnsd}
S_j x( n,  2^j( \kappa + 2^{-j_{m+1}})  )
 = \sum_{ p \atop{V_{j_{m+1}, p}\subset V_{j, n}} }
	| \lb \overline{S}_{j_{m}}  x (\cdot,  2^{j_m} \kappa ) , \psi_{j_{m + 1},  p}   \rb |.
\end{equation}
\end{proposition}
\begin{proof}
We derive from
the definition of a scattering transform in equations (3,4) in the text that
\begin{equation*}
\begin{split}
S_{j+1} x(n, 2 q) 
  &= S_j x(a_n, q) + S_j x(b_n, q) 
   = \lb \overline{S}_j x(\cdot, q ), 1_{V_{j+1}, n}\rb,  \\
S_{j+1} x(n, 2 q+ 1)
 & = | S_j x(a_n, q) - S_j x(b_n, q) | 
   =  | \lb \overline{S}_j x(\cdot, q ), \psi_{j+1, n} \rb|. 
\end{split}
\end{equation*}
where $V_{j+1,n} = V_{j, a_n} \cup  V_{j, b_n}$.
Observe that
\begin{equation*}
2^{j_{m+1}} ( \kappa + 2^{-j_{m+1}})
 =  2^{j_{m+1}} \kappa + 1
 = 2 (  2^{j_{m+1}-1} \kappa ) + 1,
\end{equation*}
thus $S_{j_{m+1}} x (n,  2^{j_{m+1}} ( \kappa + 2^{-j_{m+1}})  )$ is calculated from the coefficients $ S_{j_{m+1} -1} x ( n ,   2^{j_{m+1}-1} \kappa   )  $  of the previous layer with
\begin{equation}
\label{eqn:jm+1}
S_{j_{m+1}} x (n,   2^{j_{m+1}} ( \kappa + 2^{-j_{m+1}})   ) 
 = |\lb \overline{S}_{j_{m+1} -1} x (\cdot,   2^{j_{m+1}-1} \kappa   ), \psi_{j_{m+1}, n} \rb|.
\end{equation}
Since $2^{j+1}\kappa = 2 \cdot 2^{j}\kappa$,  the coefficient 
$S_{j_{m+1}-1} x (n,   2^{j_{m+1}-1} \kappa    )$
is calculated from $S_{j_{m}} x (n, 2^{j_m} \kappa  )$ by $(j_{m+1}-1-j_m)$ times additions, and thus
\begin{equation}
\label{eqn:fromjmtojm+1-1}
S_{j_{m+1} -1} x (n,   2^{j_{m+1}-1} \kappa   )
 = \lb \overline{S}_{j_m} x(\cdot,  2^{j_m} \kappa   ), 1_{V_{j_{m+1}-1, n}} \rb.
\end{equation}
Combining equations (\ref{eqn:fromjmtojm+1-1}) and (\ref{eqn:jm+1}) gives
\begin{equation}
\label{eqn:jm+1b}
S_{j_{m+1}} x (n,   2^{j_{m+1}} ( \kappa + 2^{-j_{m+1}}) ) 
 = |\lb \overline{S}_{j_m} x(\cdot,   2^{j_m} \kappa  ), \psi_{j_{m+1}, n} \rb|.
\end{equation}
We go from the depth $j_{m+1}$ to the depth $j \geq j_{m+1}$ by computing
\[
 S_j x (n,  2^{j} ( \kappa + 2^{-j_{m+1}})   ) 
= \lb \overline {S}_{j_{m+1}} x (\cdot,   2^{j_{m+1}} ( \kappa + 2^{-j_{m+1}})  ) , 1_{V_{j, n}} \rb .
\]
Together with (\ref{eqn:jm+1b}) it proves the equation (\ref{propsdfnsd}) of
the proposition. The summation over $p, \, V_{j_{m+1}, p}\subset V_{j, n} $ comes from the inner product $\lb 1_{V_{j_{m+1}, p}}, 1_{V_{j,n}}  \rb$.
This also proves that $\kappa + 2^{-j_{m+1}}$ is the index
of a coefficient of order $m+1$.
\end{proof}

Since $S_0 x(n,0) = x(n)$, the proposition inductively proves that the coefficients at $j$-th level $S_j x(n, 2^{j} \kappa  )$ for $j_m\le j\le J$ are of order $m$. The expression in the proposition shows that an $m+1$ order scattering coefficient at scale $2^J$ is obtained by computing the Haar wavelet coefficients of several 
order $m$ 
coefficients at the scale $2^{j_{m+1}}$, taking an absolute value, and then averaging their amplitudes over $V_{J,n}$. It thus measures the
averaged variations at the scale $2^{j_{m+1}}$ of the $m$-th order scattering coefficients.

\section{Proof of Theorem 2.2}\label{sec:reconstr}

To prove Theorem 2.2, we first define an ``interlaced pairings''. We say that 
two pairings of $V=\{1, ..., d\}$ 
\[ 
\pi^{\epsilon} = \{ a_n^{\epsilon},  b_n^{\epsilon} \}_{0 \le n < d/2} 
\]
are interlaced for $\epsilon = 0, 1$ if there exists no strict subset $\Omega$ of $V$ such that  $\pi^0$ and $\pi^{1}$ are pairing elements within $\Omega$. 
The following lemma shows that a single-layer scattering operator is invertible 
with two interlaced pairings.

\begin{lemma}\label{lemma:interlacing}
Suppose that $x \in \R^d$ takes more than $2$ different values, and two pairings $\pi^0$ and $\pi^1$ of $V=\{1, ..., d\}$ are interlaced, then $x$ can be recovered from \[
S_1 x(n, 0) =   x(a_n) + x(b_n), \quad 
S_1 x(n, 1) = | x(a_n) - x(b_n)|, \quad
0 \le n < d/2.
\]
\end{lemma}
\begin{proof}
By Eq. (2),  for a triplet ${n_1, n_2, n_3}$ if $(n_1, n_2)$ is a pair in $\pi^0$ and $(n_1, n_3)$ a pair in $\pi^1$ then the pair of values $\{ x(n_1), x(n_2)\}$ are determined (with a possible switch of the two) from 
\[
x(n_1) + x(n_2), \quad |x(n_1) - x(n_2)|
\]
and those of  $\{ x(n_1), x(n_3)\}$ are determined similarly. Then unless $x(n_1) \neq x(n_2)$ and $x(n_2) = x(n_3)$ the three values $x(n_1), x(n_2), x(n_3)$ are recovered. The interlacing condition implies that  $\pi^1$ pairs $n_2$ to an index $n_4$ which can not be $n_3$ or $n_1$. Thus, the four values of $x(n_1), x(n_2), x(n_3), x(x_4)$ are specified unless $x(n_4) = x(n_1) \neq x(n_2) = x(n_3)$.  This interlacing argument can be used to extend  to $\{1,\dots, d\}$ the set of all indices $n_i$ for which $x(n_i)$ is specified, unless $x$ takes only two values.
\end{proof}

\begin{proof}[Proof of Theorem 2.2]
Suppose that  the $2^J$ multiresolution approximations are associated 
to the $J$ hierarchical pairings $(\pi_1^{\epsilon_1}, ..., \pi_J^{\epsilon_J} )$ where $\epsilon_j \in \{0,1\}$, where for each $j$, $\pi_j^{0}$ and $\pi_j^{1}$ are two interlaced pairings of $d2^{-j}$ elements.  The sequence $(\epsilon_1, ..., \epsilon_J)$ 
is a binary vector taking $2^J$ different values. 

The constraint on the signal $x$ is that each of the intermediate scattering coefficients takes more than $2$ distinct values, which holds for $x \in \R^d$ except for a union of hyperplanes which has zero measure. Thus for almost every $x \in \R^d$, the theorem follows from applying Lemma \ref{lemma:interlacing} recursively to the $j$-th level scattering coefficients for $ J-1 \ge j \ge 0 $.
\end{proof}

\end{document}